\documentclass[letterpaper]{article} 
\usepackage{aaai25}  
\usepackage{times}  
\usepackage{helvet}  
\usepackage{courier}  
\usepackage[hyphens]{url}  
\usepackage{graphicx} 
\usepackage{natbib}  
\usepackage{caption} 
\usepackage{cuted}

\usepackage{amsmath, amssymb, amsthm}
\usepackage{booktabs}
\usepackage{mathtools}
\allowdisplaybreaks

\theoremstyle{definition}
\newtheorem{definition}{Definition}

\theoremstyle{plain}
\newtheorem{theorem}{Theorem}
\newtheorem{lemma}{Lemma}

\usepackage{algorithm}
\usepackage{algorithmic}

\usepackage{newfloat}
\usepackage{listings}
\DeclareCaptionStyle{ruled}{labelfont=normalfont,labelsep=colon,strut=off} 
\floatstyle{ruled}
\newfloat{listing}{tb}{lst}{}
\floatname{listing}{Listing}
\title{ContinuouSP: Generative Model for Crystal Structure Prediction 
\\ with Invariance and Continuity}
\author{
    \textsuperscript{\rm 1} Yuji Tone,
    \textsuperscript{\rm 1} Masatoshi Hanai,
    \textsuperscript{\rm 1} Mitsuaki Kawamura,
    \textsuperscript{\rm 1} Kenjiro Taura,
    \textsuperscript{\rm 1} Toyotaro Suzumura
}
\affiliations{
    \textsuperscript{\rm 1}The University of Tokyo,
    Tokyo, Japan\\
    tone@eidos.ic.i.u-tokyo.ac.jp,
    hanai@ds.itc.u-tokyo.ac.jp,
    mitsuaki.kawamura@riken.jp, \\
    tau@eidos.ic.i.u-tokyo.ac.jp,
    suzumura@ds.itc.u-tokyo.ac.jp
}

\usepackage{bibentry}

\begin{document}

\maketitle

\begin{abstract}

The discovery of new materials using crystal structure prediction (CSP) based on generative machine learning models has become a significant research topic in recent years. 
In this paper, we study invariance and continuity in the generative machine learning for CSP.
We propose a new model, called ContinuouSP, which effectively handles symmetry and periodicity in crystals.
We clearly formulate the invariance and the continuity, and construct a model based on the energy-based model. 
Our preliminary evaluation demonstrates the effectiveness of this model with the CSP task.
\end{abstract}

%

\section{Introduction}
Recently, there has been a growing trend toward incorporating computers into the discovery of new materials, a process that traditionally relied heavily on human intuition. 
Particularly for solid materials, recent successes using generative machine learning models provide a rational approach from the perspective of ``emulating human intuition''. 
However, theoretically, a solid material consists of an atomic spatial arrangement which is not necessarily finite, making it challenging to model correctly in a generative machine learning framework.

When discussing machine learning models for materials, \emph{invariance} and \emph{continuity} are critical aspects. 
Consider a function with materials as its domain; for this function to hold physical meaning, it must be invariant under operations such as ``10 cm translation along the x-axis'' or ``60° rotation around the y-axis''. 
Additionally, if some atoms in the material are displaced by a small amount, the resulting change should also be small, meaning that the continuity must be preserved. 
In generative models, two functions are considered: (i) the mapping from training data to the training outcome, and (ii) the probability density function determined by this outcome. 
A ``physically correct'' model can only be obtained when all these requirements are satisfied.

In this paper, we focus on crystal structure prediction (CSP), which predicts crystal structures from given atom species vector in one period. 
Recent CSP models have achieved the invariance in their probability density by employing equivariant score predictors that incorporate graph neural networks (GNNs) \cite{CDVAE,DiffCSP,SyMat,DiffCSP++}.
However, it is still difficult to achieve the continuity simultaneously.

Therefore, we develop a CSP model called \emph{ContinuouSP}, which simultaneously meets the invariance and the continuity by using an energy-based model (EBM). \footnote{Code: \url{https://github.com/packer-jp/ContinuouSP}}
We employ a modified version of the crystal graph convolutional neural networks (CGCNN) \cite{CGCNN} as the energy predictor. 
CGCNN is an invariant property prediction model that can also ensure continuity with straightforward modifications.

The main contributions of this paper are as follows:
\begin{itemize}
    \item The concepts of invariance and continuity in solid materials and periodic units of crystals are mathematically formulated.
    \item A CSP model, ContinuouSP, is designed within the EBM framework to satisfy the invariance and the continuity.
    \item Through performance evaluation, it is demonstrated that while ContinuouSP does not achieve state-of-the-art performance, it surpasses traditional machine learning approaches and performs comparably to several existing generative models.
    \item We clarify the advantages and challenges of ContinuouSP through supplementary experiments and theoretical analyses.
\end{itemize}

\section{Related Work}

\subsection{Crystal Property Prediction}
Modern machine learning schemes to predict material properties were first developed to accelerate molecular dynamics simulation~\cite{PhysRevLett.98.146401}.
For this purpose, the radial and angular distribution functions were convoluted with artificial neural networks inside each atom, and then the total energy was output.
The GNN is also employed to construct a universal model including various kinds of elements~\cite{CGCNN,SchNet,MEGNet,PotNet}.
This GNN model can handle periodicity and rotation, permutation, and translation invariance and be extended to the hypergraph to treat the bond angle directly~\cite{ALIGNN}.
Recently, a more accurate graph-transformer-based model~\cite{Matformer} and infinitely fully connected neural network for crystal systems~\cite{Crystalformer} have been proposed.
These models can be assumed as an extension of the GNN for crystals by using the attention mechanism.
This study uses these GNN models to describe the logarithmic probability because that function is expected to behave as the total energy, and we can easily modify that model to guarantee the continuity of the lattice deformation.

\subsection{CSP and Crystal Generation}
CSP and crystal generation using machine learning have traditionally developed under the predict-optimize paradigm \cite{Predict_Optimize}, based on formation energy predictors and optimization methods. On the other hand, in recent years, significant research efforts have been dedicated to exploring the potential of generative models. FTCP \cite{FTCP}, proposed alongside the implementation of a crystal generative model using VAE, serves as a reversible feature representation for crystals. Among crystal generative models utilizing GANs is CrystalGAN \cite{CrystalGAN}, which targets crystals with specific compositions and enables efficient crystal generation by partitioning the search space. Furthermore, inspired by the success of diffusion models in the field of computer vision, numerous studies have reported advances using diffusion models. One advantage of diffusion models lies in their ability to guarantee the invariance of probability density functions by adopting equivariant score predictors for translations, rotations, and permutations of atomic orders. CDVAE \cite{CDVAE}, a pioneering work in this approach, explicitly handles atomic coordinates. DiffCSP \cite{DiffCSP} employs fractional coordinates and Fourier-transformed features, while SyMat \cite{SyMat} utilizes interatomic distances. EquiCSP \cite{EquiCSP} explores the invariance concerning the permutation of lattice vector orders. DiffCSP++ \cite{DiffCSP++} further extends the DiffCSP method by incorporating space group considerations. However, we believe that, although these approaches partially ensure certain invariances, they still fall short of fully satisfying all the necessary properties that should be met.

\subsection{Continuity of Machine Learning Models}
In the field of computer vision, the continuity of machine learning models has been a subject of discussion. Specifically, it has been pointed out that using classical representations such as quaternions or Euler angles for 3D rotations in point clouds or joints results in a lack of the continuity \cite{Continuity_Rotation}. In general, for 3D rotations, representations with 4 or less dimensions are insufficient from the perspective of the continuity, and it has been shown that representations of 5 or higher dimensions should be used, along with practical representation example. Additionally, self-selecting ensembles \cite{Self_Selecting_Ensambles} provide an approach to address the topological complexity arising from the rotational symmetry of the target.

\subsection{Energy-Based Models}
Energy-based models (EBMs), which trace their origins back to the Boltzmann machine \cite{Boltzman_Machine}, continue to be actively explored across various domains. For example, \cite{EBM_OpenAI} demonstrates performance comparable to then state-of-the-art methods in tasks such as image generation and corrupted data restoration. Also, Generative PointNet \cite{Generative_PointNet} leverages EBMs for point cloud generation, achieving high performance while ensuring invariance. In the realm of physics applications, protein conformation prediction stands out as a prominent example \cite{EBM_Meta}. Moreover, EBMs form the theoretical foundation of diffusion models, and their indirect impact in this area is immeasurable.

\section{Preliminaries}

\subsection{Solid Materials}
A \emph{solid material} refers to an atomic arrangement in 3D space. Its mathematical definition is as follows:
\begin{definition}[Solid Materials]
Let $\mathbb{A}$ be the set of atomic species consisting of H, He, Li, and so on. The set of all solid materials, $\mathcal{S}$, is the set of all countable subsets of $\mathbb{A} \times \mathbb{R}^{3 \times 1}$. Namely, $(a, x) \in S$ ($\in \mathcal{S}$) represents a pair of an atomic species and a coordinate in the solid material $S$.
\end{definition}

Next, we define each kind of \emph{invariance} in solid materials. Physical properties such as the formation energy, the band gap, and the bulk modulus, which take scalar values, are examples of invariance under translation or rotation. Below, $\mathcal{Y}$ represents an arbitrary set, and elements of groups are identified with their standard representations.
\begin{definition}[Translation Invariance on Solid Materials]
A partial mapping $f: \mathcal{S} \rightharpoonup \mathcal{Y}$ is said to be translation invariant if, for any solid material $S \in \mathrm{Dom}(f)$ and any translation operation $b \in \mathrm{T}(3)$, the equation $f(\{(a, x + b) \mid (a, x) \in S\}) = f(S)$ holds.
\end{definition}
\begin{definition}[Rotation Invariance on Solid Materials]
A partial mapping $f: \mathcal{S} \rightharpoonup \mathcal{Y}$ is said to be rotation invariant if, for any solid material $S \in \mathrm{Dom}(f)$ and any rotation operation $Q \in \mathrm{O}(3)$, the equation $f(\{(a, Qx) \mid (a, x) \in S \}) = f(S)$ holds.
\end{definition}

Finally, we define the \emph{continuity} of solid materials. Continuity here means that when each atom is displaced by a small amount, the changes in the values corresponding to the solid material, such as physical properties, are also small.
\begin{definition}[Continuity on Solid Materials]
A partial mapping $f: \mathcal{S} \rightharpoonup \mathbb{R}^m$ is said to be continuous at a solid material $S \in \mathcal{S}$ if, for any $\varepsilon > 0$, there exists a $\delta > 0$ such that the following conditions hold for any solid material $S' \in \mathrm{Dom}(f)$. 
There exists a bijection $\phi: S \to S'$, $\phi(a, x) = (a', x')$, and if $a' = a$ and $\|x' - x\| < \delta$, then $\|f(S') - f(S)\| < \varepsilon$. In what follows, when we simply say that $f$ is continuous, it means it is continuous for all solid materials in its domain.
\end{definition}

\subsection {Crystals}
Let us formulate the \emph{periodic unit} commonly used to describe crystals. In general, a crystal can be represented by (i) a \emph{species vector} corresponding to the species of atoms in one period, (ii) a \emph{coordinate matrix} corresponding to the coordinates of atoms in one period, and (iii) a \emph{lattice basis} corresponding to the shape of the unit cell.

\begin{definition}[Periodic Units]
The set of all periodic units with $n$ atoms per period is denoted as $\mathcal{P}_n$. This consists of $(A, (x_1, \ldots, x_n), L) \in \mathbb{A}^n \times \mathbb{R}^{3 \times n} \times \mathbb{R}^{3 \times 3}$ such that for any $i, j \in \{1, \ldots, n\}$ and $k \in \mathbb{Z}^{3 \times 1}$ with $i \neq j$ or $k \neq 0$, $x_j + Lk \neq x_i$.
Namely, in a periodic unit $(A, X, L) \in \mathcal{P}_n$, $A$ represents the species vector, $X$ represents the coordinate matrix, and $L$ represents the lattice basis.
\end{definition}

A periodic unit can be converted into a solid material as follows:

\begin{definition}[Periodic Units to Solid Materials]
The conversion from a periodic unit to a solid material, $\mathrm{PtoS}: \bigcup_{n=1}^\infty \mathcal{P}_n \to \mathcal{S}$ is represented as follows:
\begin{multline}
\mathrm{PtoS}((a_1, \ldots, a_n), (x_1, \ldots, x_n), L) = \\ \left\{
(a_i, x_i + Lk)
\mid
i \in \{1, \ldots, n\}, k \in \mathbb{Z}^{3 \times 1}
\right\}.
\end{multline}
Through this conversion $\mathrm{PtoS}$, the entire set of crystals can be expressed as the image $\mathrm{PtoS}(\bigcup_{n=1}^\infty \mathcal{P}_n)$. Namely, a crystal is a solid material where atoms are periodically arranged according to lattice basis.
\end{definition}

It should be noted that $\mathrm{PtoS}$ is not injective. In other words, a periodic unit can be re-described without changing the crystal it indicates. The replacement of lattice basis, encompassing both cases where the number of atoms per period changes and does not change, as well as the rearrangement of atomic order, corresponds to a re-description. To address this, we introduce the concept of \emph{re-description invariance}. For example, a function that takes a periodic unit as input and calculates a physical property of the crystal it indicates should reproduce the same output even if the input periodic unit is re-described. Only when this is guaranteed, the function can be interpreted as a mapping from a crystal to its physical property.

\begin{definition}[Strong Re-description Invariance]
A mapping $g: \bigcup_{n=1}^\infty \mathcal{P}_n \to \mathcal{Y}$ is said to be strongly re-description invariant if, for any periodic units $P, P' \in \bigcup_{n=1}^\infty \allowbreak \mathcal{P}_n$, $\mathrm{PtoS}(P') = \mathrm{PtoS}(P)$ implies $g(P') = g(P)$. In this case, there exists a unique mapping $f: \mathrm{PtoS}(\bigcup_{n=1}^\infty \mathcal{P}_n) \to \mathcal{Y}$ such that $f \circ \mathrm{PtoS} = g$. When $g$ is strongly re-description invariant, the notation $g \circ \mathrm{PtoS}^{-1}$ is expediently allowed and is considered equal to $f$.
\end{definition}

Additionally, a version with a fixed number of atoms per period, which is a weaker form of the above definition, is also defined.

\begin{definition}[Weak Re-description Invariance]
A mapping $g: \bigcup_{n=1}^\infty \mathcal{P}_n \to \mathcal{Y}$ is said to be $n$-weakly re-description invariant if, for any periodic units with $n$ atoms $P, P' \in \mathcal{P}_n$, $\mathrm{PtoS}(P') = \mathrm{PtoS}(P)$ implies $g(P') = g(P)$. In this case, there exists a unique mapping $f: \mathrm{PtoS}(\mathcal{P}_n) \to \mathcal{Y}$ such that $f \circ \mathrm{PtoS}|_{\mathcal{P}_n} = g|_{\mathcal{P}_n}$. When $g$ is $n$-weakly re-description invariant, the notation $g \circ \mathrm{PtoS}|_{\mathcal{P}_n}^{-1}$ is expediently allowed and is considered equal to $f$.

\end{definition}
\subsection {Crystal Graph Convolutional Neural Networks}
The \emph{crystal graph convolutional neural networks (CGCNN)} \cite{CGCNN} is a machine learning model designed to predict physical properties of crystals. By converting periodic units into graphs, it achieves any kind of invariance.

In the process of graph construction, atoms within one periodic unit of the crystal are interpreted as nodes, and edges are drawn between nodes corresponding to atomic pairs within a cutoff distance. Each edge is associated with a feature determined by the interatomic distance. Notably, this can be multi-edges.

\begin{definition}[Graph Construction in CGCNN]
For a graph including $n$ nodes corresponding to a periodic unit $P = (A, (x_1, \ldots, x_n), L) \in \mathcal{P}_n$, 
the neighborhood $\mathrm{N}_P(i)$ of a node $i \in \{1, \ldots, n\}$ is defined as follows:
\begin{equation}
\mathrm{N}_P(i) 
= \left\{(j, k) \;\middle|\;  
\begin{aligned}
&j \in \{1, \ldots, n\}, \\
&k \in \mathbb{Z}^{3 \times 1}, \\
&\|x_j + Lk - x_i\| < D
\end{aligned}
\right\},
\end{equation}
where $D$ is the cutoff distance.
When $(j, k) \in \mathrm{N}_P(i)$, an edge is drawn between nodes $i$ and $j$, carrying a feature $e_{i, j, k} \in \mathbb{R}^{d_\mathrm{e} \times 1}$ determined by the distance 
$\|x_j + Lk - x_i\|$. The resulting graph is undirected.
\end{definition}

Next, the node features are updated through graph convolution.

\begin{definition}[Graph Convolution in CGCNN]
The feature $v_i \in \mathbb{R}^{d_\mathrm{v} \times 1}$ of a node $i$, initialized based on the atomic species, is updated as follows:
\begin{equation}
v_i \leftarrow v_i + \sum_{(j, k) \in \mathrm{N}_P(i)} \psi_\theta(v_i, v_j, e_{i, j, k}).
\end{equation}
Here, $\theta$ represents the learning parameters, and $\psi_\theta$ is a continuous mapping for each column vector of the input.
\end{definition}

The resulting node features are then aggregated using mean pooling, and the output is obtained via a multi-layer perceptron (MLP). This process constitutes the CGCNN.

\begin{definition}[CGCNN]
$\mathrm{CGCNN}_\theta: \bigcup_{n=1}^\infty \mathcal{P}_n \to \mathbb{R}$ is defined by the following steps:
\begin{enumerate}
 \item Convert the input crystal into a graph.
 \item Update node features through graph convolution.
 \item Aggregate the graph features using mean pooling.
 \item Pass the aggregated features through an MLP to obtain the output.
\end{enumerate}
\end{definition}

CGCNN defined as above satisfies all the desired invariances, although it should be noted that this does not satisfy the continuity as it stands.
\begin{theorem}[Invariance of CGCNN]
\label{thm:invariance_of_CGCNN}
$\mathrm{CGCNN}_\theta$ is strongly re-description invariant. Also, $\mathrm{CGCNN}_\theta \circ \mathrm{PtoS}^{-1}$ is invariant to translation and rotation.
\end{theorem}

\subsection {Energy-Based Models}

The \emph{energy-based model (EBM)} is one of the generative models designed to learn the underlying probability distribution of a given dataset. Unlike other generative models such as VAEs or GANs, the NN in an EBM does not directly output samples. Instead, it outputs a scalar value called \emph{energy} associated with the input, which is tied to the probability density.

\begin{definition}[Probability Density Function in EBMs]
The goal is to learn a probability distribution over a set $\mathcal{X}$. For a point $x \in \mathcal{X}$, the parameterized probability density is denoted as $p_\theta(x)$. A relationship is established between $p_\theta(x)$ and an NN $H_\theta: \mathcal{X} \to \mathbb{R}$ as follows:
\begin{gather}
p_\theta(x) = \frac{\exp(-\beta H_\theta(x))}{Z(\theta, \beta)}, \\
\text{where } Z(\theta, \beta) = \int_{x \in \mathcal{X}} \exp(-\beta H_\theta(x)) \mathrm{d}x.
\end{gather}
Here, $\beta > 0$ is the inverse temperature. By analogy with statistical mechanics, $H_\theta$ is referred to as the \emph{energy function}.
\end{definition}

A sampling method using the energy function is Markov-chain Monte Carlo (MCMC). In MCMC, a point is initially chosen at random and then iteratively updated. A fundamental MCMC algorithm is the Metropolis-Hastings (MH) algorithm, which proposes transitions to neighboring states and decides whether to accept these transitions based on changes in energy.

\begin{definition}[MH]
In MH, by denoting the transition probability from a point $x \in \mathcal{X}$ to another point $x' \in \mathcal{X}$ as $r(x' \mid x)$, the acceptance probability is defined as follows:
\begin{equation}
\min\left\{1, \frac{r(x \mid x') p_\theta(x')}{r(x' \mid x) p_\theta(x)}\right\}.
\end{equation}
\end{definition}

This ensures that the detailed balance condition is satisfied, and after sufficient iterations, the samples follow the probability distribution $p_\theta$.

If $H_\theta$ is differentiable with respect to the input, another MCMC method, Langevin Monte Carlo (LMC) algorithm, can be applied.

\begin{definition}[LMC]
In LMC, a point $x \in \mathcal{X} = \mathbb{R}^m$ is updated as follows:
\begin{equation}
x \leftarrow x - \alpha \beta \nabla_x H_\theta(x) + \sqrt{2\alpha} u,
\text{where } u \sim \mathcal{N}(0, I)
\end{equation}
If the step size $\alpha$ is sufficiently small, no acceptance decision like in MH is required. An LMC variant that includes acceptance decisions is called the Metropolis-adjusted Langevin algorithm (MALA).
\end{definition}

The EBM is typically trained via maximum likelihood estimation.

\begin{definition}[Loss Function in EBMs]
In an EBM, given a dataset $(x_1, \ldots, x_n) \in \mathcal{X}^n$, learning minimizes the following loss $J(\theta)$:
\begin{equation}
J(\theta) = -\sum_{i=1}^n \log p_\theta(x_i).
\end{equation}
\end{definition}

Although directly computing the loss function is difficult, its gradient can be calculated.

\begin{theorem}[Gradient of Loss Function in EBMs]
\label{thm:gradient_of_loss_function_in_EBMs}
The gradient of the loss $J(\theta)$ is given by:
\begin{equation}
\nabla_\theta J(\theta) 
= \frac{\beta}{n} \sum_{i=1}^n \left( \nabla_\theta H_\theta(x_i) - \mathbb{E}_{p_\theta(x)} \left[ \nabla_\theta H_\theta(x) \right] \right).
\end{equation}
Here, the subscript $p_\theta(x)$ of the expectation operator $\mathbb{E}$ represents sampling $x$ from the probability distribution corresponding to $p_\theta$.
\end{theorem}

In essence, learning progresses to minimize the energy for data points while maximizing the energy for sample points.

\section{Proposed Method}

\begin{figure*}[tb]
\centering
\includegraphics[width=0.95\textwidth]{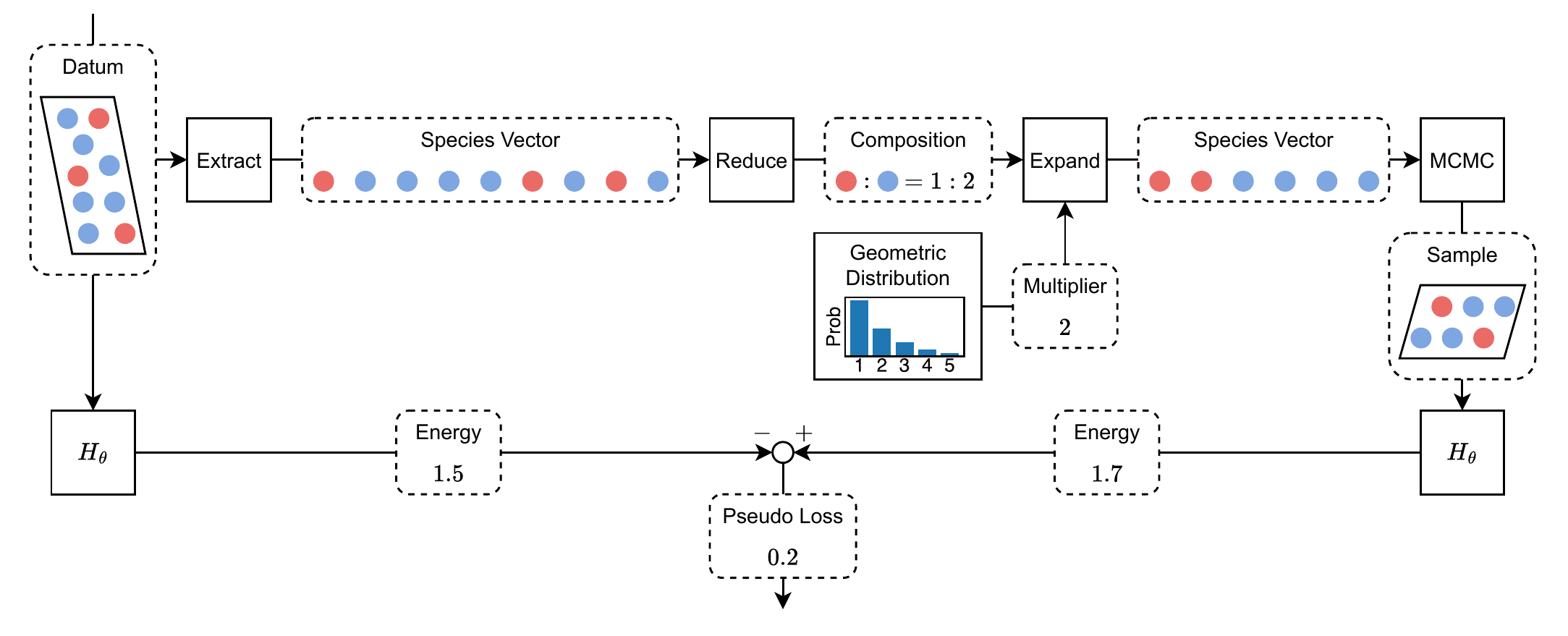} 
\caption{Diagram illustrating the training workflow of ContinuouSP: For each periodic unit of the crystal included in the training data points, the species vector is $\mathrm{Reduce}$-d to the composition, then $\mathrm{Expand}$-ed using a geometric distribution to create a new species vector. This vector is used to condition the sampling performed via MCMC. The resulting new crystal and the original crystal are compared by calculating the energy difference, which is used as a pseudo-loss.}
\label{fig:training}
\end{figure*}

We propose a CSP model, \emph{ContinuouSP}, that is both invariant and continuous. By leveraging the EBM framework, we can retain the invariance properties guaranteed by CGCNN.

First, we modify CGCNN slightly to ensure the continuity.

\begin{definition}[Graph Convolution in CGCNN']
$\mathrm{CGCNN}_\theta'$ is defined as a modified version of the update formula of $\mathrm{CGCNN}_\theta$, as follows:
\begin{multline}
v_i \leftarrow v_i + \\ \sum_{(j, k) \in \mathrm{N}_P(i)} \cos^2\Bigg(\frac{\pi\|x_j + Lk - x_i\|}{2D}\Bigg) \psi_\theta(v_i, v_j, e_{i, j, k}).
\end{multline}
Here, $\psi_\theta$ is a continuous mapping for each column vector of its inputs.
\end{definition}

Next, we use this directly as the energy function.

\begin{definition}[Energy Function in ContinuouSP]
The energy function $H_\theta: \bigcup_{n=1}^\infty \mathcal{P}_n \to \mathbb{R}$ is defined as $H_\theta = \mathrm{CGCNN}_\theta'$
\end{definition}

For the CSP task, which is a conditional generation problem based on the species vector, we define the relationship between the energy and the probability density.

\begin{definition}[Probability Density Function in ContinuouSP]
For a periodic unit $(A, X, L) \in \mathcal{P}_n$, the conditional probability density $p_\theta(X, L \mid A)$ is expressed using the energy function $H_\theta$ as:
\begin{gather}
\label{eq:probability_distribution_function_in_ContinuouSP}
p_\theta(X, L \mid A) = \frac{\exp(-\beta H_\theta(A, X, L))}{Z(\theta, \beta, A)}, \\
\begin{multlined}
\label{eq:partition_function_in_ContinuouSP}
\text{where } Z(\theta, \beta, A) = \\ \int_{\substack{X \in \mathbb{R}^{3 \times n} \\ L \in \mathbb{R}^{3 \times 3}}} \exp(-\beta H_\theta(A, X, L)) \mathrm{d}X \mathrm{d}L. 
\end{multlined}
\end{gather}
\end{definition}

The energy function $H_\theta$ satisfies the following properties:

\begin{theorem}[Properties of Energy Function in ContinuouS\-P]
\label{thm:properties_of_energy_function_in_ContinuouSP}
$H_\theta$ is strongly re-description invariant. Also, $H_\theta \circ \mathrm{PtoS}^{-1}$ is invariant to translation and rotation, and is continuous.
\end{theorem}

\begin{theorem}[Properties of Probability Density Function in ContinuouSP]
\label{thm:properties_of_probability_density_function_in_ContinuouSP}
For any positive integer $n$, $p_\theta$ is $n$-weakly re-description invariant. Also, $p_\theta \circ \mathrm{PtoS}|_{\mathcal{P}_n}^{-1}$ is invariant to translation and rotation, and is continuous.
\end{theorem}

In the standard EBM, the parameters $\theta$ are determined via maximum likelihood estimation over the dataset. However, in the CSP task, the generation condition, species vector, is not inherently unique to the crystal (even ignoring permutation symmetry). Thus, directly using it in the loss function is inappropriate. In ContinuouSP, the loss function is defined prescriptively as follows:

\begin{definition}[Loss Function in ContinuouSP]
The loss function $J$ in ContinuouSP on the dataset $((A_1, X_1, L_1), \allowbreak \ldots, (A_m, X_m, L_m)) \in \left(\bigcup_{n=1}^\infty \mathcal{P}_n\right)^m$ is defined as:
\begin{multline}
J(\theta) = \frac{1}{m} \sum_{i=1}^m \Bigg(\beta H_\theta(A_i, X_i, L_i) + \\ \sum_{j=1}^\infty (1-q)^{j-1} q \log Z(\theta, \beta, A_i')\Bigg),
\end{multline}
where: 
\begin{itemize}
    \item $(1-q)^{j-1}q$ is the probability mass of geometric distribution with success probability $q$ and $j$ trials.
    \item $A_i'=\mathrm{Expand}(\mathrm{Reduce}(A_i), j)$.
    \item $\mathrm{Reduce}$ is a function that constructs the simplest integer ratio of atomic counts (so-called composition) from the species vector.
    \item $\mathrm{Expand}$ is a function that constructs the species vector by multiplying the composition by an integer.
\end{itemize}
Due to the re-description invariance of $H_\theta$, the computation outcome does not depend on the order of atoms in $A_i'$. Therefore, this loss function is well-defined.
\end{definition}

Then, the gradient of the loss function can be expressed in a form that corresponds well to that of a general EBMs.
\begin{theorem}[Gradient of Loss Function in ContinuouSP]
\label{thm:gradient_of_loss_function_in_ContinuouSP}
The gradient of the loss $J(\theta)$ is given by:
\begin{multline}
\nabla_\theta J(\theta) = \frac{\beta}{m} \sum_{i=1}^m \Bigg( \nabla_\theta H_\theta(A_i, X_i, L_i) - \\ \mathbb{E}_{\substack{j \sim \mathrm{Geom}(q) \\ p_\theta(X, L \mid A_i')}}\left[ \nabla_\theta H_\theta(A_i', X, L) \right] \Bigg), 
\end{multline}
where $\mathrm{Geom}(q)$ is geometric distribution with success probability $q$.
\end{theorem}

Thus, the pseudo-loss is computed through the flow illustrated in Fig. \ref{fig:training}, guiding the training process.

\begin{definition}[Loss Function for Single Datum]
For the mapping \(\mathcal{L}: \bigcup_{n=1}^\infty \mathcal{P}_n \to \mathbb{R}\) from a single datum to the loss, the following holds:
\begin{multline}
\mathcal{L}(A, X, L) = \frac{1}{m}\Bigg(\beta H_\theta(A, X, L) + \\ \sum_{j=1}^\infty (1-q)^{j-1} q \log Z(\theta, \beta, A')\Bigg) + R(\theta),
\end{multline}
where $R(\theta)$ accounts for the contributions of other data points.
\end{definition}

The following properties hold for $\mathcal{L}$:

\begin{theorem}[Properties of Loss Function in ContinuouSP]
\label{thm:properties_of_loss_function_in_ContinuouSP}
The mapping $\mathcal{L}$ is strongly re-description invariant. Also, the mapping $\mathcal{L} \circ \mathrm{PtoS}^{-1}$ is invariant to translation and rotation, and is continuous.
\end{theorem}

If the mapping from a datum to the loss is invariant and continuous, it is reasonable to expect that the training results will also be invariant and continuous with respect to the datum. Therefore, we conclude that in ContinuouSP, both the training results regarding the training data and the probability density functions determined by the training results satisfy all requirements.

\section{Experiment}

\subsection{Experimental Setup}
\subsubsection{Crystal Structure Prediction}

We evaluate the performance of ContinuouSP on the CSP task using datasets that were utilized in previous researches such as \cite{CDVAE} and \cite{DiffCSP}. After training as usual, we generate species vectors for each crystal in the test dataset, use them as input conditions for the model, and compare the actual structures with the sampled structures.

\paragraph{Datasets}
We evaluate our methods on three datasets with varying complexity. \emph{Perov-5} \cite{Perov_5_0}\cite{Perov_5_1} consists of 18,928 crystals with perovskite structures. Each crystal has a cubic unit cell containing 5 atoms, characterized by similar structural configurations. This dataset is relatively simple. \emph{MP-20} comprises 45,231 inorganic crystals and is considered a standard dataset. Each crystal contains up to 20 atoms per unit cell. \emph{MPTS-52} is a more complex extension of MP-20, consisting of 40,476 inorganic crystals. Each crystal also contains up to 20 atoms per unit cell. Perov-5 and MP-20 datasets are split into 60-20-20 ratios for training, validation, and testing. For MPTS-52, the dataset is split into 27,380/5,000/8,096 entries based on the order of the earliest published year.

\paragraph{Baselines}
We compare our approach with several baseline methods to evaluate its effectiveness. \emph{RS}, \emph{BO}, and \emph{PSO} \cite{Predict_Optimize} represent traditional methods based on the predict-optimize paradigm. Specifically, a formation energy predictor is first constructed using MEGNet \cite{MEGNet}, followed by the discovery of structures that minimize formation energy through various optimization methods. Each of these methods employs over 5,000 iteration steps. \emph{P-cG-SchNet} \cite{DiffCSP} is an autoregressive generative model adapted for CSP by incorporating lattice basis predictions into SchNet \cite{SchNet}, which was originally designed for molecular tasks. \emph{CDVAE} \cite{CDVAE} is a crystal generative model based on diffusion and can be applied to CSP. \emph{DiffCSP} \cite{DiffCSP} is another diffusion-based model explicitly designed for CSP.

\paragraph{Metrics}
For evaluation, 20 samples are generated for each crystal in the test dataset, and each kind of metric value is computed by using \texttt{StructureMatcher(stol=0.5, angle\_tol=10, ltol=0.3)} from pymatgen library. \emph{Match Rate} is the percentage of test crystals for which at least one sampled structure matches the actual structure. A match is determined if the \texttt{get\_rms\_dist} method of \texttt{StructureMatcher} returns a value other than \texttt{None}. \emph{RMSE} is the mean of the smallest atomic RMS displacements among matched structures for each test crystal. The \texttt{float} values returned by \texttt{get\_rms\_dist} method are used for this calculation.

\paragraph{Implementation Details}
In the $\psi_\theta$ of CGCNN, the sum of the outputs of linear layers applied to $v_i$ and $v_j$, and the Hadamard product of the output of a linear layer applied to $e_{i, j, k}$, are prepared twice in the same way. One of them is passed through a sigmoid function, and the two are multiplied. After each graph convolutional layer, the vertex features are passed through a softplus activation function. The edge features are computed using Gaussian smearing with a maximum value of 20 Å and a standard deviation coefficient of 3.0. The cutoff distance $D$ for graph construction is set to 3.0 times the cube root of the reciprocal of the atomic density. For Perov-5, both the vertex feature dimension $d_\mathrm{v}$ and the edge feature dimension $d_\mathrm{e}$ are set to 32, with 3 graph convolutional layers and a 2-layer MLP with the softplus activation function. On the other hand, for MP-20 and MPTS-52, both $d_\mathrm{v}$ and $d_\mathrm{e}$ are set to 64, with 6 graph convolutional layers and a 4-layer MLP also with the softplus activation function. To prevent excessively low or high atomic densities, a penalty is added to the energy output by the neural network, which is the logarithm of the atomic density divided by the reference atomic density of 0.05 $\text{\AA}^{-3}$. The geometric distribution parameter $q$ in the loss function is set to 0.5, and sampling is performed using 1000 iterations of MALA. During the iterations, the inverse temperature $\beta$ changes exponentially from 1 to 1000, and the step size $\alpha$ changes exponentially from 0.5 to 0.0005. The lattice basis is initialized with a normal distribution such that the expected atomic density is 0.05 $\text{\AA}^{-3}$, and the atomic coordinates are initialized with a uniform distribution in fractional coordinates. At each sampling step, the lattice basis is reduced using Niggli reduction. The training optimization algorithm uses Adam with a learning rate of 0.001, $\beta_1 = 0.9$, and $\beta_2 = 0.999$. Training is performed for 5 epochs on Perov-5 and 1 epoch on MP-20 and MPTS-52, using a single NVIDIA A100 GPU.
 
\subsubsection{Displacement vs. Energy}
As a supplementary experiment to verify the validity of the model trained for the CSP task, the change in energy with respect to atomic displacement from stable structures is examined. Specifically, for crystals with stable structures included in the test dataset, Gaussian noise with a standard deviation of 0.1 Å is added to the coordinates of a single atom within the periodic unit, and the resulting structure is passed through the energy predictor. By repeating this process a sufficient number of times, the relation between the displacement and the energy can be obtained. For this experiment, five crystals are randomly selected from each of Perov-5, MP-20, and MPTS-52. From the periodic unit of each selected crystal, one atom is randomly chosen. Then, 1000 pairs of displacement magnitudes and energy values are obtained.

\subsection{Experimental Result}
\subsubsection{Crystal Structure Prediction}

\begin{table*}[t]
\centering
\caption{The results of performance evaluation on the CSP task: Metrics such as Match Rate and RMSE are reported for datasets like Perov-5, MP-20, and MPTS-50, alongside results from existing methods based on the predict-optimize paradigm and generative models. ContinuouSP generally outperforms methods based on the predict-optimize paradigm and demonstrates performance comparable to several existing generative models.}
\label{tab:result_csp}
\setlength{\tabcolsep}{10pt}
\renewcommand{\arraystretch}{1.5}
\begin{tabular}{c c c c c c c}
\toprule
 & \multicolumn{2}{c}{Perov-5} & \multicolumn{2}{c}{MP-20} & \multicolumn{2}{c}{MPTS-52} \\
 & Match rate$\uparrow$ & RMSE$\downarrow$ & Match rate$\uparrow$ & RMSE$\downarrow$ & Match rate$\uparrow$ & RMSE$\downarrow$ \\
\midrule \midrule
RS & 29.22 & 0.2924 & 8.73 & 0.2501 & 2.05 & 0.3329 \\
BO & 21.03 & 0.2830 & 8.11 & 0.2816 & 2.05 & 0.3024 \\
PSO & 20.90 & 0.0836 & 4.05 & 0.1670 & 1.06 & 0.2339 \\
\midrule \midrule
P-cG-SchNet & 97.94 & 0.3463 & 32.64 & 0.3018 & 12.96 & 0.3942 \\
CDVAE & 88.51 & 0.0464 & 66.95 & 0.1026 & 20.79 & 0.2085 \\
DiffCSP & 98.60 & \textbf{0.0128} & \textbf{77.93} & \textbf{0.0492} & \textbf{34.02} & \textbf{0.1749} \\
\midrule \midrule
ContinuouSP & \textbf{99.16} & 0.0893 & 44.21 & 0.1757 & 11.77 & 0.2647 \\
\bottomrule
\end{tabular}
\end{table*}

The experimental results are presented in Table \ref{tab:result_csp}. Overall, while the ContinuouSP does not surpass DiffCSP, generally outperforms the methods based on the predict-optimize paradigm, and exhibits comparable performance to several existing generative models. Notably, the Match Rate on the Perov-5 dataset achieves the highest value among all methods. It is unfortunate that the relative performance decreases as the complexity of the datasets increases. However, this is primarily attributed to the current inability to identify optimal hyperparameter settings for these cases, rather than any fundamental issues in the core methodology causing such tendencies. Regarding training time, it remained within 10 hours for all datasets. This is neither significantly longer nor shorter compared to prior models. Nevertheless, the drawback of EBMs --- the long time required for a single epoch — has become apparent. Addressing this issue, such as by employing parallelizable sampling methods, will be a focus for future improvements.

\subsubsection{Displacement vs. Energy}

\begin{figure*}[tb]
\centering
\includegraphics[width=0.95\textwidth]{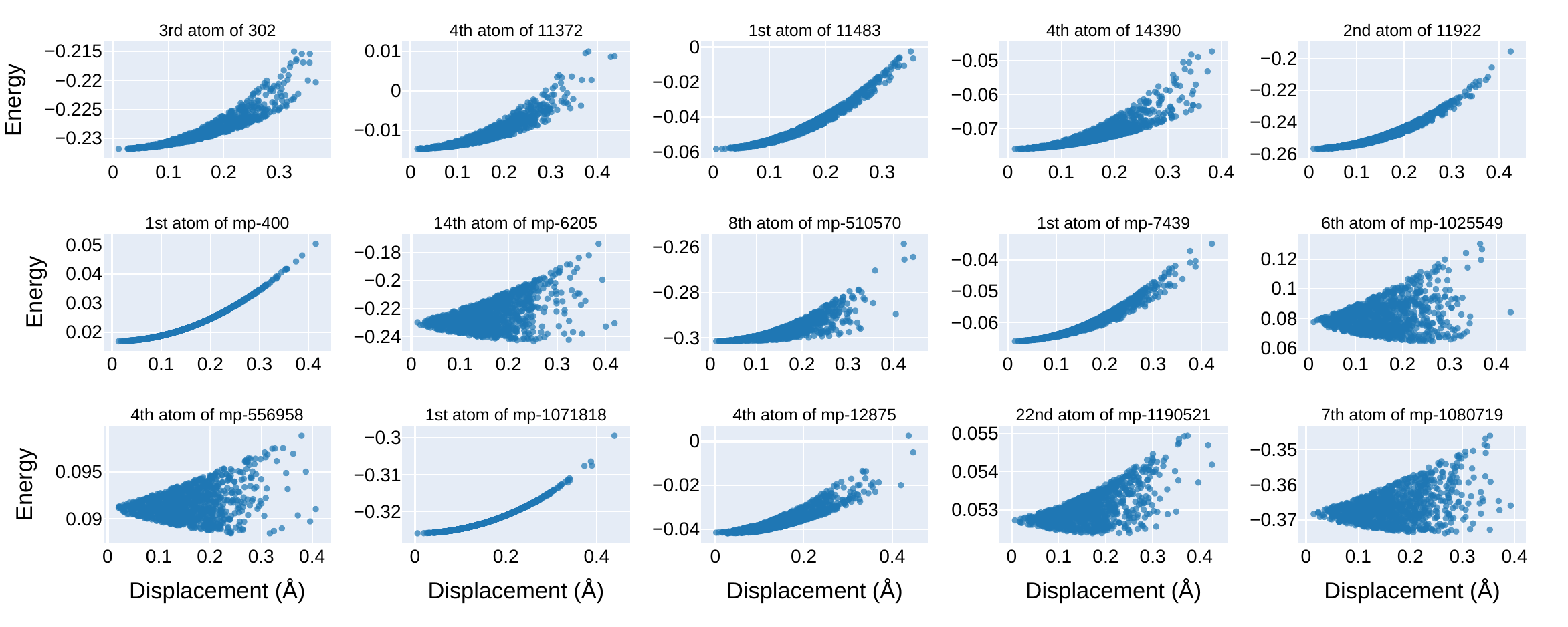} 
\caption{The results of the displacement vs. energy experiment: For crystals with stable structures, an atom is selected from the periodic unit, and Gaussian noise is added to its coordinates. The diagram shows the relationship between the displacement and the energy output by the model. From top to bottom, the results correspond to crystals picked from the Perov-5, MP-20, and MPTS-52 datasets, respectively. Note that the format of IDs representing crystals varies across datasets. The relationship between displacement and energy is divided into two types: those where energy monotonically increases with displacement and those where both increases and decreases are observed.}
\label{fig:displacement_vs_energy}
\end{figure*}

We show the experimental result in Fig. \ref{fig:displacement_vs_energy}. Across all datasets, the behavior can be categorized into two patterns: one where the energy increases monotonically with displacement from the stable position, and another where both increases and decreases in energy are observed. For the cases of monotonic energy increase, the slope also increases from horizontal as displacement grows, which aligns well with physical behavior, such as the harmonic force field approximation. Notably, these phenomena are reproduced despite the fact that no explicit constraints, such as those used in loss functions, are incorporated into the model, as is the case like \cite{CDVAE}. There are also instances where a range of energy values is observed for the same displacement magnitude, which can be interpreted as the model capturing the anisotropy of energy changes. On the other hand, cases where energy decreases with displacement indicate instances where the model fails to correctly identify the energy of the stable structure as a true local minimum. Thus, while our model can automatically emulate correct physical behavior in some cases, it also faces challenges in consistently handling all scenarios.

\section{Discussion}
\subsection{Relation with Data Augmentation}
One method to address arbitrariness in a training dataset is data augmentation. For instance, this involves applying random rotations to image data before feeding it into the model. So far, we have discussed the perspective of ensuring various invariances in crystals at the model level. However, ContinuouSP actually incorporates elements of data augmentation. Specifically, the process of doubling the number of atoms per period through $\mathrm{Reduce}$ and $\mathrm{Expand}$ serves this purpose, absorbing the arbitrariness related to the size of the period. While addressing other forms of arbitrariness through data augmentation could have been a viable approach, ensuring the invariances at the model level is preferable. This is because data augmentation struggles to handle the replacement of lattice basis that does not change the number of atoms per period. Unlike rotation or translational symmetry in periodicity, lattice basis replacement is unbounded and cannot be sampled correctly due to its symmetry and the inherent arbitrariness unrelated to the number of atoms per period. Taking these factors into consideration, we aimed to maximize the invariances at the model level by utilizing an EBM with CGCNN as its backbone.

\subsection{Comparison with Other Methods}
Here, we compare ContinuouSP with other methods from a theoretical perspective.

First, CSP using EBMs is similar to the predict-optimize paradigm. This is because the neural network in an EBM functions as an energy predictor, and the behavior of MCMC closely resembles optimization---particularly at low temperatures, where it can essentially be regarded as optimization. However, the reason we adopted EBM, despite the similarities, lies in the inherent challenges of the predict-optimize paradigm. Specifically, the paradigm's reliance on constructing an energy predictor from actual formation energy data makes it less robust to data biases. The pairs of crystal structures and formation energies in the training data must conform to the probability distribution derived from those energies. If the training data consists only of stable crystals, for instance, the energy predictor might overestimate the energy of unstable structures during optimization. Preparing a suitable dataset that satisfies this requirement is not straightforward. In contrast, EBMs or general generative models learn the probability distribution based solely on the presence or absence of data. This means that a dataset consisting only of stable crystals would suffice. Nonetheless, constructing formation energy predictors or other property predictors and utilizing them for transfer learning in EBMs or similar models could be promising for performance improvement.

The relationship with diffusion models is also crucial. Diffusion models can be seen as an evolution of EBMs, and their efficiency arises from not requiring MCMC sampling during training. Instead, they explicitly assume a probability distribution from the training data and train a score predictor to learn its gradient. We attribute the success of diffusion models in the computer vision domain partly to the validity of assuming simple distributions, such as mixtures of Gaussians, for objects in Euclidean space like images. In contrast, the topological space of crystals is more complex due to periodicity. This complexity is exemplified by the fact that even for the same crystal, the degrees of freedom can vary depending on the method of description per period. From these observations, we consider the use of EBMs instead of diffusion models in CSP to be reasonably justified. However, from the perspective of training efficiency, diffusion models have a clear advantage, as reflected in the performance evaluation results of this study. Improving the training efficiency of EBMs remains an important challenge for future work.

\section{Conclusion}
In this study, we proposed the CSP model ContinuouSP. We began by formalizing key concepts, including invariance to translation and rotation and continuity, both at the solid material level, as well as the conversion from periodic units to solid materials and the associated invariance to re-description. Using the property prediction model CGCNN and the EBM framework, we successfully fulfilled these requirements. In performance evaluation, while the model did not achieve state-of-the-art on ordinary datasets, it outperformed traditional methods and demonstrated comparable performance to several existing generative models. Moreover, a unique feature of ContinuouSP was its ability to maintain performance on extraordinary datasets, a characteristic not observed in existing models. Additionally, we discussed the relationship with data augmentation, principled comparisons with existing methods such as the predict-optimize paradigm and diffusion models, and the theoretical underpinnings of these approaches. Through this work, we reaffirm the potential of EBM in CSP tasks. Future work will focus on improving performance to further solidify its advantages.




\section{Acknowledgments}
This research was conducted using the Wisteria / BDEC-01 at The University of Tokyo.
This work was supported by MEXT ``Advanced Research Infrastructure for Materials and Nanotechnology in Japan (ARIM) '', MEXT ``Developing a Research Data Ecosystem for the Promotion of Data-Driven Science'' (2024-10) and JSPS KAKENHI (22K17899, 20K15012).
We used ChatGPT (GPT-4o) for grammar checks and writing revisions.

\bibliography{aaai25}

\begin{thebibliography}{24}
\providecommand{\natexlab}[1]{#1}

\bibitem[{Behler and Parrinello(2007)}]{PhysRevLett.98.146401}
Behler, J.; and Parrinello, M. 2007.
\newblock Generalized Neural-Network Representation of High-Dimensional Potential-Energy Surfaces.
\newblock \emph{Phys. Rev. Lett.}, 98: 146401.

\bibitem[{Castelli et~al.(2012{\natexlab{a}})Castelli, Landis, Thygesen, Dahl, Chorkendorff, Jaramillo, and Jacobsen}]{Perov_5_0}
Castelli, I.~E.; Landis, D.~D.; Thygesen, K.~S.; Dahl, S.; Chorkendorff, I.; Jaramillo, T.~F.; and Jacobsen, K.~W. 2012{\natexlab{a}}.
\newblock New cubic perovskites for one-and two-photon water splitting using the computational materials repository.
\newblock \emph{Energy \& Environmental Science}, 5(10): 9034--9043.

\bibitem[{Castelli et~al.(2012{\natexlab{b}})Castelli, Olsen, Datta, Landis, Dahl, Thygesen, and Jacobsen}]{Perov_5_1}
Castelli, I.~E.; Olsen, T.; Datta, S.; Landis, D.~D.; Dahl, S.; Thygesen, K.~S.; and Jacobsen, K.~W. 2012{\natexlab{b}}.
\newblock Computational screening of perovskite metal oxides for optimal solar light capture.
\newblock \emph{Energy \& Environmental Science}, 5(2): 5814--5819.

\bibitem[{Chen et~al.(2019)Chen, Ye, Zuo, Zheng, and Ong}]{MEGNet}
Chen, C.; Ye, W.; Zuo, Y.; Zheng, C.; and Ong, S.~P. 2019.
\newblock Graph Networks as a Universal Machine Learning Framework for Molecules and Crystals.
\newblock \emph{Chemistry of Materials}, 31(9): 3564--3572.

\bibitem[{Cheng, Gong, and Yin(2022)}]{Predict_Optimize}
Cheng, G.; Gong, X.-G.; and Yin, W.-J. 2022.
\newblock Crystal structure prediction by combining graph network and optimization algorithm.
\newblock \emph{Nature Communications}, 13(1): 1492.

\bibitem[{Choudhary and DeCost(2021)}]{ALIGNN}
Choudhary, K.; and DeCost, B. 2021.
\newblock Atomistic Line Graph Neural Network for improved materials property predictions.
\newblock \emph{npj Computational Materials}, 7(1): 185.

\bibitem[{Du et~al.(2020)Du, Meier, Ma, Fergus, and Rives}]{EBM_Meta}
Du, Y.; Meier, J.; Ma, J.; Fergus, R.; and Rives, A. 2020.
\newblock Energy-based models for atomic-resolution protein conformations.
\newblock In \emph{International Conference on Learning Representations}.

\bibitem[{Du and Mordatch(2019)}]{EBM_OpenAI}
Du, Y.; and Mordatch, I. 2019.
\newblock Implicit Generation and Generalization in Energy-Based Models.
\newblock \emph{CoRR}, abs/1903.08689.

\bibitem[{Hinton and Sejnowski(1986)}]{Boltzman_Machine}
Hinton, G.~E.; and Sejnowski, T.~J. 1986.
\newblock \emph{Learning and relearning in Boltzmann machines}, 282–317.
\newblock Cambridge, MA, USA: MIT Press.
\newblock ISBN 026268053X.

\bibitem[{Jiao et~al.(2023)Jiao, Huang, Lin, Han, Chen, Lu, and Liu}]{DiffCSP}
Jiao, R.; Huang, W.; Lin, P.; Han, J.; Chen, P.; Lu, Y.; and Liu, Y. 2023.
\newblock Crystal Structure Prediction by Joint Equivariant Diffusion.
\newblock In Oh, A.; Naumann, T.; Globerson, A.; Saenko, K.; Hardt, M.; and Levine, S., eds., \emph{Advances in Neural Information Processing Systems}, volume~36, 17464--17497. Curran Associates, Inc.

\bibitem[{Jiao et~al.(2024)Jiao, Huang, Liu, Zhao, and Liu}]{DiffCSP++}
Jiao, R.; Huang, W.; Liu, Y.; Zhao, D.; and Liu, Y. 2024.
\newblock Space Group Constrained Crystal Generation.
\newblock In \emph{The Twelfth International Conference on Learning Representations}.

\bibitem[{Lin et~al.(2024)Lin, Chen, Jiao, Mo, Jianhuan, Huang, Liu, Huang, and Lu}]{EquiCSP}
Lin, P.; Chen, P.; Jiao, R.; Mo, Q.; Jianhuan, C.; Huang, W.; Liu, Y.; Huang, D.; and Lu, Y. 2024.
\newblock Equivariant Diffusion for Crystal Structure Prediction.
\newblock In Salakhutdinov, R.; Kolter, Z.; Heller, K.; Weller, A.; Oliver, N.; Scarlett, J.; and Berkenkamp, F., eds., \emph{Proceedings of the 41st International Conference on Machine Learning}, volume 235 of \emph{Proceedings of Machine Learning Research}, 29890--29913. PMLR.

\bibitem[{Lin et~al.(2023)Lin, Yan, Luo, Liu, Qian, and Ji}]{PotNet}
Lin, Y.; Yan, K.; Luo, Y.; Liu, Y.; Qian, X.; and Ji, S. 2023.
\newblock Efficient Approximations of Complete Interatomic Potentials for Crystal Property Prediction.
\newblock In Krause, A.; Brunskill, E.; Cho, K.; Engelhardt, B.; Sabato, S.; and Scarlett, J., eds., \emph{Proceedings of the 40th International Conference on Machine Learning}, volume 202 of \emph{Proceedings of Machine Learning Research}, 21260--21287. PMLR.

\bibitem[{Luo, Liu, and Ji(2023)}]{SyMat}
Luo, Y.; Liu, C.; and Ji, S. 2023.
\newblock Towards Symmetry-Aware Generation of Periodic Materials.
\newblock In Oh, A.; Naumann, T.; Globerson, A.; Saenko, K.; Hardt, M.; and Levine, S., eds., \emph{Advances in Neural Information Processing Systems}, volume~36, 53308--53329. Curran Associates, Inc.

\bibitem[{Nouira, Sokolovska, and Crivello(2019)}]{CrystalGAN}
Nouira, A.; Sokolovska, N.; and Crivello, J.-C. 2019.
\newblock CrystalGAN: Learning to Discover Crystallographic Structures with Generative Adversarial Networks.
\newblock arXiv:1810.11203.

\bibitem[{Ren et~al.(2022)Ren, Tian, Noh, Oviedo, Xing, Li, Liang, Zhu, Aberle, Sun, Wang, Liu, Li, Jayavelu, Hippalgaonkar, Jung, and Buonassisi}]{FTCP}
Ren, Z.; Tian, S. I.~P.; Noh, J.; Oviedo, F.; Xing, G.; Li, J.; Liang, Q.; Zhu, R.; Aberle, A.~G.; Sun, S.; Wang, X.; Liu, Y.; Li, Q.; Jayavelu, S.; Hippalgaonkar, K.; Jung, Y.; and Buonassisi, T. 2022.
\newblock An invertible crystallographic representation for general inverse design of inorganic crystals with targeted properties.
\newblock \emph{Matter}, 5(1): 314--335.

\bibitem[{Sch\"{u}tt et~al.(2017)Sch\"{u}tt, Kindermans, Sauceda~Felix, Chmiela, Tkatchenko, and M\"{u}ller}]{SchNet}
Sch\"{u}tt, K.; Kindermans, P.-J.; Sauceda~Felix, H.~E.; Chmiela, S.; Tkatchenko, A.; and M\"{u}ller, K.-R. 2017.
\newblock SchNet: A continuous-filter convolutional neural network for modeling quantum interactions.
\newblock In Guyon, I.; Luxburg, U.~V.; Bengio, S.; Wallach, H.; Fergus, R.; Vishwanathan, S.; and Garnett, R., eds., \emph{Advances in Neural Information Processing Systems}, volume~30. Curran Associates, Inc.

\bibitem[{Taniai et~al.(2024)Taniai, Igarashi, Suzuki, Chiba, Saito, Ushiku, and Ono}]{Crystalformer}
Taniai, T.; Igarashi, R.; Suzuki, Y.; Chiba, N.; Saito, K.; Ushiku, Y.; and Ono, K. 2024.
\newblock Crystalformer: Infinitely Connected Attention for Periodic Structure Encoding.
\newblock In \emph{The Twelfth International Conference on Learning Representations}.

\bibitem[{Xiang(2021)}]{Self_Selecting_Ensambles}
Xiang, S. 2021.
\newblock Eliminating topological errors in neural network rotation estimation using self-selecting ensembles.
\newblock \emph{ACM Trans. Graph.}, 40(4).

\bibitem[{Xie et~al.(2021)Xie, Xu, Zheng, Gao, Wang, Song-Chun, and Wu}]{Generative_PointNet}
Xie, J.; Xu, Y.; Zheng, Z.; Gao, R.; Wang, W.; Song-Chun, Z.; and Wu, Y.~N. 2021.
\newblock Generative PointNet: Deep Energy-Based Learning on Unordered Point Sets for 3D Generation, Reconstruction and Classification.
\newblock In \emph{The IEEE Conference on Computer Vision and Pattern Recognition (CVPR)}.

\bibitem[{Xie et~al.(2022)Xie, Fu, Ganea, Barzilay, and Jaakkola}]{CDVAE}
Xie, T.; Fu, X.; Ganea, O.-E.; Barzilay, R.; and Jaakkola, T.~S. 2022.
\newblock Crystal Diffusion Variational Autoencoder for Periodic Material Generation.
\newblock In \emph{International Conference on Learning Representations}.

\bibitem[{Xie and Grossman(2018)}]{CGCNN}
Xie, T.; and Grossman, J.~C. 2018.
\newblock Crystal Graph Convolutional Neural Networks for an Accurate and Interpretable Prediction of Material Properties.
\newblock \emph{Phys. Rev. Lett.}, 120: 145301.

\bibitem[{Yan et~al.(2022)Yan, Liu, Lin, and Ji}]{Matformer}
Yan, K.; Liu, Y.; Lin, Y.; and Ji, S. 2022.
\newblock Periodic Graph Transformers for Crystal Material Property Prediction.
\newblock In \emph{The 36th Annual Conference on Neural Information Processing Systems}.

\bibitem[{Zhou et~al.(2019)Zhou, Barnes, Jingwan, Jimei, and Hao}]{Continuity_Rotation}
Zhou, Y.; Barnes, C.; Jingwan, L.; Jimei, Y.; and Hao, L. 2019.
\newblock On the Continuity of Rotation Representations in Neural Networks.
\newblock In \emph{The IEEE Conference on Computer Vision and Pattern Recognition (CVPR)}.

\end{thebibliography}

\newpage

\section{Appendix}
\subsection{Invariance of CGCNN}
Let us prove Theorem \ref{thm:invariance_of_CGCNN}. First, in order to prove the strong re-description invariance of $\mathrm{CGCNN}_\theta$, consider constructing a graph corresponding to a solid material.
\begin{definition}[Graph Construction for Solid Materials]
For a graph including infinite nodes corresponding to a solid material $S \in \mathcal{S}$, the neighborhood $\mathrm{N}_S(a, x)$ of a vertex $(a, x) \in S$ is defined as follows:
$$
\mathrm{N}_S(a, x) = \{(a', x') \mid (a', x') \in S, \|x' - x\| < D\}.
$$
For $(a', x') \in \mathrm{N}_S(a, x)$, an edge is formed between vertices $(a, x)$ and $(a', x')$, with a feature $e_{a, x, a', x'} \in \mathbb{R}^{d_\mathrm{e} \times 1}$ determined by the distance $\|x' - x\|$. Due to symmetry, the resulting graph is undirected.
\end{definition}

\begin{definition}[Graph Convolution for Solid Materials]
The feature $u_{a, x} \in \mathbb{R}^{d_\mathrm{v} \times 1}$ of a vertex $(a, x)$ is initialized based on its atomic species and is updated as follows:
$$
u_{a, x} \leftarrow u_{a, x} + \sum_{(a', x') \in \mathrm{N}_S(a, x)} \psi_\theta(u_{a, x}, u_{a', x'}, e_{a, x, a', x'}).
$$
\end{definition}

By using these, the strong re-description invariance of $\mathrm{CGCNN_\theta}$ can be established as follows:

\begin{lemma}[Strong Re-description Invariance of CGCNN]
\label{lem:strong_redescription_invariance_of_CGCNN}
$\mathrm{CGCNN_\theta'}$ is strongly re-description invarant.
\end{lemma}

\begin{proof}[Proof of Lemma \ref{lem:strong_redescription_invariance_of_CGCNN}]
Consider a periodic unit $((a_1, \ldots, a_n), \allowbreak (x_1, \ldots, x_n), L) \in \mathcal{P}_n$ corresponding to a solid material $S = \mathrm{PtoS}(P)$. By mathematical induction, after graph construction and a fixed number of graph convolution steps, for any $i \in \{1, \ldots, n\}$ and $k \in \mathbb{Z}^{3 \times 1}$, we have $u_{a_i, x_i + Lk} = v_{i, k}$. The resulting node features are passed through mean pooling and an MLP, yielding the same result.
\end{proof}

Next, we show proofs of the translation and rotation invariance of $\mathrm{CGCNN}_\theta \circ \mathrm{PtoS}^{-1}$

\begin{lemma}[Translation invariance of CGCNN]
\label{lem:translation_invariance_of_CGCNN}
$\mathrm{CGCNN_\theta} \circ \mathrm{PtoS}^{-1}$ is translation invariant.
\end{lemma}
\begin{lemma}[Rotation invariance of CGCNN]
\label{lem:rotation_invariance_of_CGCNN}
$\mathrm{CGCNN_\theta} \circ \mathrm{PtoS}^{-1}$ is rotation invariant.
\end{lemma}

\begin{proof}[Proof of Lemma \ref{lem:translation_invariance_of_CGCNN}]
Let a crystal be $C \in \mathrm{PtoS}(\bigcup_{n=1}^\infty \mathcal{P}_n)$ and a periodic unit $P = (A, (x_1, \ldots, x_n), L)$ such that $\mathrm{PtoS}(P) = C$. For any translation $b \in \mathrm{T}(3)$, define a new periodic unit $P' = (A, (x_1 + b, \ldots, x_n + b), L)$. The graph and initial node features constructed from $P$ and $P'$ are identical. Thus, $\mathrm{CGCNN}_\theta(P') = \mathrm{CGCNN}_\theta(P)$.
Hence:
\begin{align*}
&\mathrm{CGCNN}_\theta \circ \mathrm{PtoS}^{-1}(\{(a, x + b) \mid (a, x) \in C\}) \\
= &\mathrm{CGCNN}_\theta(P') \\
= &\mathrm{CGCNN}_\theta(P) \\
= &\mathrm{CGCNN}_\theta \circ \mathrm{PtoS}^{-1}(C).
\end{align*}
\end{proof}
\begin{proof}[Proof of Lemma \ref{lem:rotation_invariance_of_CGCNN}]
It can be proven in a similar manner to Lemma \ref{lem:translation_invariance_of_CGCNN}.
\end{proof}

Using the above, Theorem \ref{thm:invariance_of_CGCNN} follows.
\begin{proof}[Proof of Theorem \ref{thm:invariance_of_CGCNN}]
It follows from Lemmas \ref{lem:strong_redescription_invariance_of_CGCNN}, \ref{lem:translation_invariance_of_CGCNN}, \ref{lem:rotation_invariance_of_CGCNN}.
\end{proof}

\subsection{Gradient of Loss function in EBMs}
Theorem \ref{thm:gradient_of_loss_function_in_EBMs} is well-known and can be derived through the following transformations:
\begin{proof}[Proof of Theorem \ref{thm:gradient_of_loss_function_in_EBMs}]
\begin{align}
    & \nabla_\theta J(\theta) \\
    = & \frac{\beta}{n} \sum_{i=1}^n \nabla_\theta H_\theta(x_i) + \nabla_\theta \log Z(\theta, \beta) \\
    = & \frac{\beta}{n} \sum_{i=1}^n \nabla_\theta H_\theta(x_i) + \frac{\nabla_\theta Z(\theta, \beta)}{Z(\theta, \beta)} \\
    = & \frac{\beta}{n} \sum_{i=1}^n \nabla_\theta H_\theta(x_i) - \notag \\
    &\qquad\frac{\beta}{Z(\theta, \beta)} \int_{x \in \mathcal{X}} \exp(-\beta H_\theta(x)) \nabla_\theta H_\theta(x) \mathrm{d}x \\
    = & \frac{\beta}{n} \sum_{i=1}^n \nabla_\theta H_\theta(x_i) - \beta \mathbb{E}_{p_\theta(x)}\left[\nabla_\theta H_\theta(x)\right] \\
    = & \frac{\beta}{n} \sum_{i=1}^n (\nabla_\theta H_\theta(x_i) - \mathbb{E}_{p_\theta(x)}\left[\nabla_\theta H_\theta(x)\right]).
\end{align}
\end{proof}

\subsection{Properties of Energy Function in ContinuouSP}
Let us prove Theorem \ref{thm:properties_of_energy_function_in_ContinuouSP}. By $H_\theta = \mathrm{CGCNN}_\theta'$, it suffices to demonstrate the properties of $\mathrm{CGCNN}_\theta'$. Since most of the proof is almost identical to that of Theorem \ref{thm:invariance_of_CGCNN}, we focus here on proving the continuity of $\mathrm{CGCNN}_\theta' \circ \mathrm{PtoS}^{-1}$. First, consider the following lemma:
\begin{lemma}[continuity on periodic units to solid materials]
\label{lem:continuity_on_periodic_units_to_solid_materials}
Let $g: \bigcup_{n=1}^\infty \mathcal{P}_n \to \mathbb{R}^m$ be strongly permutation-invariant. Then, if the mapping $h_A$ satisfying $h_A(X, L) = g(A, X, L)$, is continuous for any $A$ with respect to each column vector of the input, then $g \circ \mathrm{PtoS}^{-1}$is continuous.
\end{lemma}
\begin{proof}[Proof of Lemma \ref{lem:continuity_on_periodic_units_to_solid_materials}]
Let $C \in \mathrm{PtoS}(\bigcup_{n=1}^\infty \mathcal{P}_n)$ and $\varepsilon > 0$. Consider $P = (A, X, L) = (A, (x_1, \ldots, x_n), (l_1, l_2, l_3))$ such that $\mathrm{PtoS}(P) = C$ and $n$ has sufficiently many divisors. By the continuity of $h_A$, there exists $\delta > 0$ such that for any periodic unit $P' = (A, X', L')$, if the norms of each column of $X' - X$ and $L' - L$ are less than $\delta$, then $\|h_A(X', L') - h_A(X, L)\| < \varepsilon$. For this $\delta$, consider any $C' \in \mathrm{PtoS}(\bigcup_{n=1}^\infty \mathcal{P}_n)$. There exists a bijection $\phi: C \to C'$ such that for any atom $(a, x) \in C$, if $\phi(a, x) = (a', x')$, then $a' = a$ and $\|x' - x\| < \delta$. Thus, there exists $P' = (A, X', L') = (A, (x_1', \ldots, x_n'), (l_1', l_2', l_3'))$ such that $\mathrm{PtoS}(P') = C'$ and for any $i \in \{1, \ldots, n\}$ and $k, k' \in \mathbb{Z}^{3 \times 1}$, $\|x_i' + L'k' - x_i - Lk\| < \delta$. By considering $k = k' = 0$, $\|x_i' - x_i\| < \delta$ stands. Also, by considering $k = k' = (\pm 1, 0, 0)^\mathrm{T}$, $\|l_1' - l_1\| < \delta$ stands. Similarly, $\|l_2' - l_2\| < \delta$ and $\|l_3' - l_3\| < \delta$. Hence, $\|h_A(X', L') - h_A(X, L)\| = \|g(P') - g(P)\| = \|g \circ \mathrm{PtoS}^{-1}(C') - g \circ \mathrm{PtoS}^{-1}(C)\| < \varepsilon$.
\end{proof}

By using this lemma, we can prove the continuity of $\mathrm{CGCNN}_\theta' \circ \mathrm{PtoS}^{-1}$ as follows:
\begin{lemma}[Continuity of CGCNN']
\label{lem:continuity_of_CGCNN'}
$\mathrm{CGCNN}_\theta' \circ \mathrm{PtoS}^{-1}$ is continuous.
\end{lemma}
\begin{proof}[Proof of Lemma \ref{lem:continuity_of_CGCNN'}]
The mapping $\omega_k$ defined as 
\begin{multline}
\omega_k(x_i, x_j, L) = \\
\left\{\begin{array}{ll}
\cos^2\left(\frac{\pi\|x_j + Lk - x_i\|}{2D}\right) & (\|x_j + Lk - x_i\| < D) \\
0 & (\|x_j + Lk - x_i\| \geq D)
\end{array}\right.
\end{multline}
is continuous with respect to each column vector of the input.Using this, the update rule for $\mathrm{CGCNN}_\theta'$ is expressed as:
\begin{equation}
v_i \leftarrow v_i + \sum_{\substack{j \in \{1, \ldots, n\} \\ k \in \mathbb{Z}^{3 \times 1}}} \omega_k(x_i, x_j, L)\psi_\theta(v_i, v_j, e_{i, j, k}).
\end{equation}
This update rule relies solely on mappings that are continuous with respect to each column vector of $X$ and $L$. Additionally, since mean pooling and MLP are continuous mappings, $\mathrm{CGCNN}_\theta'$ satisfies the conditions of Lemma \ref{lem:continuity_on_periodic_units_to_solid_materials}.
\end{proof}

Thus, the proof of Theorem \ref{thm:properties_of_energy_function_in_ContinuouSP} is complete.
\begin{proof}[Proof of Theorem \ref{thm:properties_of_energy_function_in_ContinuouSP}]
The content corresponding to Theorem \ref{thm:invariance_of_CGCNN} and Lemma \ref{lem:continuity_of_CGCNN'} implies properties concerning $\mathrm{CGCNN_\theta'}$. These properties also hold for $H_\theta = \mathrm{CGCNN_\theta'}$.
\end{proof}

\subsection{Properties of Probability Density Function in ContinuouSP}
\begin{proof}[Proof of Theorem \ref{thm:properties_of_probability_density_function_in_ContinuouSP}]
Considering the substitution that simultaneously changes the order of $X'$ in Eq. \ref{eq:partition_function_in_ContinuouSP} along with the order of $A$, the value of $Z$, the denominator in Eq. \ref{eq:probability_distribution_function_in_ContinuouSP}, remains unchanged due to the strong re-description invariance of $H_\theta$. Therefore, revisiting the strong re-description invariance of $H_\theta$ in the numerator of Eq. \ref{eq:probability_distribution_function_in_ContinuouSP}, we can conclude that $p_\theta$ remains unchanged even when the unit cell is altered without changing the crystal or the number of atoms per period. Furthermore, due to the invariance and the continuity of $H_\theta \circ \mathrm{PtoS}^{-1}$, it follows immediately that $p_\theta \circ \mathrm{PtoS}|_{\mathcal{P}_n}^{-1}$ also exhibits the same properties.
\end{proof}

\subsection{Gradient of Loss Function in ContinuouSP}
\begin{proof}[Proof of Theorem \ref{thm:gradient_of_loss_function_in_ContinuouSP}]
\begin{align}
    & \nabla_\theta J(\theta) \\
    =&\frac{1}{m} \sum_{i=1}^m \Bigg(\beta \nabla_\theta H_\theta(A_i, X_i, L_i) + \notag \\ 
    & \qquad\qquad\quad \sum_{j=1}^\infty (1-q)^{j-1} q \nabla_\theta \log Z(\theta, \beta, A_i')\Bigg) \\
    = &\frac{1}{m} \sum_{i=1}^m \Bigg(\beta \nabla_\theta H_\theta(A_i, X_i, L_i) + \notag \\
    & \qquad\qquad\qquad\sum_{j=1}^\infty (1-q)^{j-1} q \frac{\nabla_\theta Z(\theta, \beta, A_i')}{Z(\theta, \beta, A_i')} \Bigg) \\
    = &\frac{1}{m} \sum_{i=1}^m \Bigg(\beta \nabla_\theta H_\theta(A_i, X_i, L_i) + \sum_{j=1}^\infty \frac{\beta(1-q)^{j-1} q}{Z(\theta, \beta, A_i')} \notag \\
    &\!\!\! \int_{\substack{X \in \mathbb{R}^{3 \times n} \\ L \in \mathbb{R}^{3 \times 3}}} \exp(-\beta H_\theta(A_i', X, L)) \nabla_\theta H_\theta(A_i', X, L) \mathrm{d}X \mathrm{d}L \Bigg) \\
    = & \frac{1}{m} \sum_{i=1}^m \Bigg(\beta \nabla_\theta H_\theta(A_i, X_i, L_i) - \notag \\
    & \qquad\qquad\qquad \beta\mathbb{E}_{\substack{j \sim \mathrm{Geom}(q) \\ p_\theta(X, L \mid A_i')}}\left[ \nabla_\theta H_\theta(A_i', X, L) \right] \Bigg) \\
    = & \frac{\beta}{m} \sum_{i=1}^m \Bigg( \nabla_\theta H_\theta(A_i, X_i, L_i) - \notag \\ 
    & \qquad\qquad\qquad \mathbb{E}_{\substack{j \sim \mathrm{Geom}(q) \\ p_\theta(X, L \mid A_i')}}\left[ \nabla_\theta H_\theta(A_i', X, L) \right] \Bigg).
\end{align}
\end{proof}

\subsection{Properties of Loss Function in ContinuouSP}
\begin{proof}[Proof of Theorem \ref{thm:properties_of_loss_function_in_ContinuouSP}]
Since $\mathrm{Reduce}(A)$ is the composition which is intrinsic to the crystal, the summation doesn't depend on re-description. Thus, also with the strong re-description invariance of $H_\theta$, the strong re-description invariance of $\mathcal{L}$ stands. Furthermore, due to the invariance and the continuity of $H_\theta \circ \mathrm{PtoS}^{-1}$, it follows immediately that $\mathcal{L} \circ \mathrm{PtoS}^{-1}$ also exhibits the same properties.
\end{proof}

\end{document}